\documentclass[sigconf]{acmart}

\usepackage[ruled]{algorithm2e} 

\SetAlFnt{\small}
\SetAlCapFnt{\small}
\SetAlCapNameFnt{\small}
\SetAlCapHSkip{0pt}
\IncMargin{-\parindent}
\usepackage{booktabs} 

\usepackage{multirow}
\usepackage{subcaption}
\usepackage{balance}

\usepackage[hyphenbreaks]{breakurl}

\begin{document}

\title{Budget Constrained Bidding by Model-free Reinforcement Learning in Display Advertising}

\author{Di Wu, Xiujun Chen, Xun Yang, Hao Wang, Qing Tan, Xiaoxun Zhang, Jian Xu, Kun Gai}
 \affiliation{%
   \institution{Alibaba Group}
   \city{Beijing}
   \country{P.R.China}
 }
 \email{{di.wudi,xiujun.cxj,vincent.yx,wh111044,qing.tan,xiaoxun.zhang,xiyu.xj,jingshi.gk}@alibaba-inc.com}

%






\copyrightyear{2018} 
\acmYear{2018} 
\setcopyright{acmlicensed}
\acmConference[CIKM '18]{The 27th ACM International Conference on Information and Knowledge Management}{October 22--26, 2018}{Torino, Italy}
\acmBooktitle{The 27th ACM International Conference on Information and Knowledge Management (CIKM '18), October 22--26, 2018, Torino, Italy}
\acmPrice{15.00}
\acmDOI{10.1145/3269206.3271748}
\acmISBN{978-1-4503-6014-2/18/10}

\begin{abstract}
Real-time bidding (RTB) is an important mechanism in online display advertising, where a proper bid for each page view plays an essential role for good marketing results. 
Budget constrained bidding is a typical scenario in RTB where the advertisers hope to maximize the total value of the winning impressions under a pre-set budget constraint.
However, the optimal bidding strategy is hard to be derived due to the complexity and volatility of the auction environment. 
To address these challenges, in this paper, we formulate budget constrained bidding as a Markov Decision Process and propose a model-free reinforcement learning framework to resolve the optimization problem. Our analysis shows that the immediate reward from environment is misleading under a critical resource constraint. Therefore, we innovate a reward function design methodology for the reinforcement learning problems with constraints. Based on the new reward design, we employ a deep neural network to learn the appropriate reward so that the optimal policy can be learned effectively. Different from the prior model-based work, which suffers from the scalability problem, our framework is easy to be deployed in large-scale industrial applications. The experimental evaluations demonstrate the effectiveness of our framework on large-scale real datasets.
\end{abstract}

%
%



\maketitle

\section{Introduction}
In recent years, online display advertising has become one of the most influential businesses, with
\$17.6 billion revenues in HY 2017 in US alone\cite{iab-hy-2017}. Real-Time Bidding (RTB) \cite{wang2015real,yuan2013real} is probably the most important mechanism in online display advertising. In RTB, the advertisers have the ability to bid for an ad impression opportunity and only the highest bidder wins the opportunity to display its ad. The winning advertiser has the privilege to enjoy the value of the ad impression but with a certain cost. The impression value is usually associated with the expectation of the desired outcomes such as ad clicks or conversions. The cost is usually determined based on the specific auction mechanism adopted. In this paper, without loss of generality, we focus our discussion under the second price auction \cite{edelman2007internet} where the winning advertiser is charged the second highest bid in the auction. Our approach is also applicable under other auction mechanisms such as Vickrey-Clarke-Groves auction (VCG) \cite{nisan2007algorithmic}.

In RTB, a typical optimization goal for the advertisers is to maximize the total value of winning impressions under a certain budget. Budget constrained bidding is an automated bidding strategy to address this need. The advertisers can simply set their optimization goal such as maximizing total ad clicks or conversions with a certain budget and the bidding strategy is able to calculate the bid for each ad impression opportunity on behalf of the advertisers to achieve their optimization goal. This bidding strategy, which significantly simplifies marketing strategy and improves marketing efficiency, is widely provided by many global advertising platforms such as Google, Facebook, and Alibaba.

Budget constrained bidding is essentially a knapsack problem\cite{lin2016combining}, and Zhang et al. \cite{zhang2016optimal} gave the theoretical proof that the optimal bid takes the form of $v/\lambda$ under the second price auction mechanism, where $v$ is the impression value and $\lambda$ is a scaling parameter mapping $v$ to a scale of bid. This bidding formula has a very straightforward interpretation: the larger the impression value is, the higher the bid should be offered, and the impression value is scaled by a parameter $\lambda$ to derive the bid. A high $\lambda$ may cause a low bid, which might be lower than the optimal one, so that budget may not be fully spent. On the contrary, a low $\lambda$ may make the bid higher than the optimal one and budget may run out early so that potentially more valuable impressions may be missed. Unfortunately, it is extremely difficult to obtain the optimal $\lambda$ in the RTB environment. First, there are usually thousands or even more heterogeneous bidders competing for the same ad opportunities which makes the marketplace highly dynamic and unpredictable. Second, the advertisers themselves may change campaign settings such as budget and targeted audience. 

There are two categories of existing work on resolving the budget constrained bidding problem. The first category makes use of the optimal bidding formula and dynamically tunes the parameter $\lambda$. 
One solution is to use $\lambda=f(budget)$, and this comes from the observation that the value of $\lambda$ determines the budget spending speed, and thus budget spending speed should be used as a signal to adjust $\lambda$. 
However, how to determine the best budget spending speed is still an open question \cite{zhou2008budget}. The second category formulates the auction process as a \textit{Markov Decision Process} (MDP) \cite{sutton1998reinforcement}, where an action in the MDP is to generate a bid. Based on this formulation, Cai et al. \cite{cai2017real} tried to use reinforcement learning (RL) algorithms to solve the budget constrained bidding problem. However, model-based RL approaches such as \cite{cai2017real} require storing the state transition matrix and using dynamic programming algorithms, whose computational cost is unacceptable in real-world advertising platforms.

In this paper, we propose a novel approach for budget constrained bidding by leveraging model-free reinforcement learning. More specifically, we train an agent to sequentially regulate the bidding parameter $\lambda$ instead of directly producing bids. The agent strives to learn and adapt to the highly non-stationary environment in order to make $\lambda$ always close to the optimal one. The appealing properties of the approach is not only sticking to the theoretical optimal bidding strategy but also avoiding expensive computational cost brought by model-based RL approaches. 


During the training process, we found that simply using immediate reward will make the agent easily converge to suboptimal solutions. This is mainly because the budget constraint is neglected in the reward design. The immediate reward will simply make agent obsessed with taking actions to decrease $\lambda$ to gain more immediate reward at each step. To address this challenge, we innovate a reward design methodology which leads the agent to efficiently converge to the optimal solution. 
We also designed an adaptive $\epsilon$-greedy policy to adjust the exploration probability based on how the state-action value is distributed. We conducted experiments on large-scale real dataset, and observed improvements on the key metric compared with the state-of-the-art RL-based bidding strategies. Our main contributions can be summarized as follows:

\begin{enumerate}
\item To the best of our knowledge, our work presents the first model-free RL approach for the budget constrained bidding problem. The approach demonstrates significant advantage over existing heuristic or model-based RL approaches.
\item A novel reward function design methodology is proposed to lead the budget constrained bidding agent to its optimum efficiently. This methodology can also be applied to other scenarios involving resource-constrained RL.
\item We invent a deep reinforcement learning to bid framework that puts together the capabilities of deep neural networks, model-free RL, and our policy optimization innovations. The framework is already deployed and validated in industry-scale advertising systems.
\end{enumerate}

The rest of this paper is organized as follows. Section 2 introduces the budget constrained bidding problem and reinforcement learning basics. In section 3, we present our model-free RL approach towards the budget constrained bidding problem and give detailed analysis on how to design reward function and improve $\epsilon$-greedy policy. Section 4 discusses the experimental evaluation results, followed by related work in section 5. Section 6 concludes the paper. 

\section{Background}

\subsection{Budget Constrained Bidding} \label{problem_formulation}

Budget constrained bidding is a typical strategy in RTB where the advertisers hope to maximize the total value of winning impressions under a budget. The bidding process of an advertiser can be described as follows. During a time period, one day for instance, there are $N$ impression opportunities arriving sequentially ordered by an index $i$. For simplicity, we slightly abuse the notations and use the index $i$ to represent the $i$-th impression opportunity. The advertiser gives a bid $b_i$ according to the impression value $v_i$ and competes with other bidders in real-time. If $b_i$ is the highest in the auction, the advertiser has the privilege to display its ad and enjoys the impression value. Winning an impression also associates with a cost $c_i$. In the second price auction, cost is determined by the second highest bid in the auction. The bidding process terminates whenever the total cost reaches the advertiser's budget limit $B$ or all the impression opportunities have gone through the auction. Let $x_i$ be a binary indicator whether the advertiser wins impression $i$, the goal of budget constrained bidding is to maximize the total value of winning impressions under the budget:

\begin{equation} \label{bcb_f}
\begin{aligned}
\max \displaystyle\sum\limits_{i=1...N}x_i v_i \ \ \\
s.t. \sum_{i=1}^N x_i c_i \leq B
\end{aligned}
\end{equation}

Lin et al. \cite{lin2016combining} formalized budget constrained bidding as a knapsack problem, and Zhang et al. \cite{zhang2016optimal, zhang2014optimal} proved that under the second price auction, the optimal bidding strategy takes the form of
\begin{equation} 
b_i=v_i/\lambda \label{eq:1}
\end{equation}
where $\lambda$ is a scaling factor. When the impression opportunity sequence is known apriori, the optimal $\lambda$, say $\lambda^*$, can be derived through greedy approximation algorithm \cite{dantzig1957discrete}. Unfortunately, the strategy needs to bid in real time without knowing the candidate impressions, and the auction environment is usually highly non-stationary due to the dynamics of all the participating bidders, which makes $\lambda^*$ hard to obtain. 

\subsection{Reinforcement Learning and Constrained Markov Decision Process}
Reinforcement learning (RL) is a machine learning approach inspired by behaviorist psychology. In RL, an agent interacts with its environment by sequentially taking actions, observing consequences, and altering its behaviors in order to maximize a cumulative reward. RL is usually modeled as a \textit{Markov Decision Process} (MDP) which consists of: a state space $\mathcal{S}=\{s\}$, an action space $\mathcal{A}=\{a\}$, state transition dynamics $\mathcal{T}: \mathcal{S}\times\mathcal{A}\rightarrow\mathcal{P}(\mathcal{S})$ where $\mathcal{P}(\mathcal{S})$ is the set of probability measures on $\mathcal{S}$, an immediate reward function $r: \mathcal{S}\times\mathcal{A} \rightarrow \mathbb{R}$, and a discount factor $\gamma\in [0,1]$. A policy, denoted by $\pi: \mathcal{S}\rightarrow\mathcal{P}(\mathcal{A})$ where $\mathcal{P}(\mathcal{A})$ is the set of probability measures on $\mathcal{A}$, fully defines the behavior of an agent. The agent uses its policy to interact with the environment and gives a trajectory of states, actions, and rewards $s_1, a_1, r_1, ..., s_T, a_T, r_T$ ($T=\infty$ indicates a infinite horizon MDP and otherwise an episodic one) over $\mathcal{S}\times\mathcal{A}\times \mathbb{R}$. The cumulative discounted reward constitutes the return $R=\sum_{t=1}^T\gamma^{t-1}r_t$. The agent's goal is to learn an optimal policy $\pi^*$ that maximizes the expected return from the start state.
\begin{equation} \label{eq:policy}
\begin{aligned}
\pi^*=\arg\max_{\pi} \mathbb{E}[R|\pi]
\end{aligned}
\end{equation}

Interacting with the environment may also incur costs $c^k: \mathcal{S}\times\mathcal{A}\rightarrow\mathbb{R}, k\in\{1,...,K\}$ where $c^k$ is the cost of type $k$. Let $C^k=\sum_{t=1}^T\gamma^{t-1}c_t^k$ be the cumulative discounted cost of type $k$. \textit{Constrained Markov Decision Process} (CMDP) \cite{altman1999constrained} extends MDP by introducing additional constraints on the costs so that the policy optimization problem can be written as

\begin{equation} \label{eq:policy}
\begin{aligned}
\max_{\pi}\ &\mathbb{E}[R|\pi] \\
s.t.\ \mathbb{E}(C^k|\pi) \leq &\widehat{C}^k\ \ \forall\ k\in\{1,...,K\}
\end{aligned}
\end{equation}

\noindent where $\widehat{C}^k$ is the constraint of type $k$ cost. 

\section{Method}
It is intuitive to model budget constrained bidding with CMDP where the agent's action is submitting bids to sequential impression opportunities and the constraint is the total budget. However, this straightforward method is not feasible in the real applications. First, policy optimization in CMDP is usually solved via model-based RL approaches and more specifically via linear programming \cite{altman1999constrained,huang2014constrained,haskell2013stochastic,du2017improving}. These approaches need to know (or predict) the transition dynamics $\mathcal{T}$ apriori. However, the RTB environment is highly non-stationary and unpredictable (as argued in Section 1) which makes the transition dynamics hard to obtain. Second, these approaches are also computationally costly and therefore are not applicable in the RTB environment where there are typically billions of impression opportunities to bid for an advertiser on a daily basis. 

To tackle these challenges and inspired by the optimal bidding theory \cite{zhang2016optimal}, we propose to model budget constrained bidding as a $\lambda$ control problem in the CMDP framework and solve it via model-free RL. More specifically, we employ deep Q-network (DQN) \cite{mnih2015human}, a variant of Q-learning as our RL algorithm. Q-learning iteratively updates the action-value function $Q(s_t,a_t)$ to quantify the quality of taking action $a_t$ at state $s_t$ and DQN uses deep neural networks to represent this function. We also find that the immediate reward function will make the agent easily converge to suboptimal solutions.
This is mainly because the immediate reward function neglects the budget constraint. To address this challenge, we innovate a reward function design methodology for model-free RL to solve the CMDP problem, which leads the agent to converge to the optimal solution efficiently. Besides, we also improve the $\epsilon$-greedy policy in our scenario, helping the agent balance exploitation and exploration.

\subsection{Modeling}

\begin{figure} 
  \centering
    \includegraphics[width=0.48\textwidth]{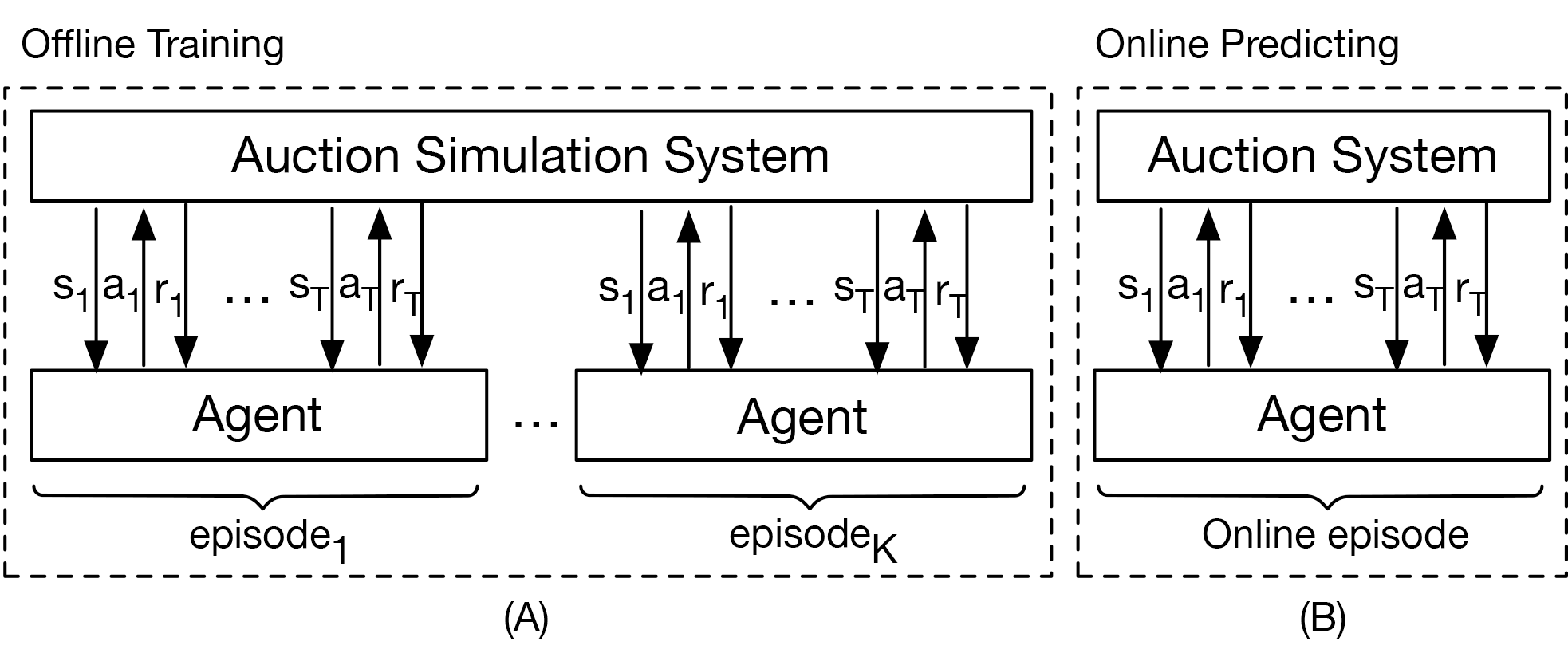}
  \caption{Illustration of $\lambda$ control process in budget constrained bidding. (A) Agent training process. (B) Agent online predicting process.}
  \label{fg:modelling}
\end{figure}

We consider an episodic CMDP with discount factor $\gamma=1$ where an episode (typically one day) starts with a budget $B$ and an initial $\lambda_0$. As shown in Fig. \ref{fg:modelling}, the agent regulates $\lambda$ sequentially with a fixed number of $T$ steps (typically 15min between two consecutive steps) until the episode ends. At each time step $t\in\{1,...,T\}$, the agent observes state $s_t$ and takes an action $a_t$ to adjust $\lambda_{t-1}$ to $\lambda_t$. The bid for any impression opportunity $i$ between time step $t$ and $t+1$ is decided by $v_i/\lambda_t$. The total value of the winning impressions between $t$ and $t+1$ constitutes the immediate reward $r_{t}$, associated with a cost $c_{t}$. The goal of the agent is to learn an optimal $\lambda$ control policy to maximize the cumulative reward $\sum_{t=1}^Tr_t$ as long as $\sum_{t=1}^Tc_t \leq B$. More specifically, the core elements of the CMDP are further explained as follows:

\begin{itemize}
\item $\mathcal{S}$: The state should in principle reflect the RTB environment and the bidding status. We consider the following statistics to represent state $s_t$: 1) $t$: the current time step, 2) $B_t$: the remaining budget at time-step $t$, 3) $ROL_t$: the number of $\lambda$ regulation opportunities left at step $t$, 4) $BCR_t=(B_t-B_{t-1})/B_{t-1}$: the budget consumption rate, 5) $CPM_t$: the cost per mille of impressions of the winning impressions between $t-1$ and $t$, 6) $WR_t$: the auction win rate reflecting the ratio of winning impressions versus total impression opportunities, and 7) $r_{t-1}$: the total value of winning impressions such as the total clicks or conversions at time-step $t-1$. 
\item $\mathcal{A}$: We design a number of adjustment rates to $\lambda$ so that an action $a\in\mathcal{A}$ typically takes the form of $\lambda_t = \lambda_{t-1}\times(1+\beta_a)$ where $\beta_a$ is the adjustment rate associated with $a$. 
\item $\mathcal{T}$: Since we take the model-free RL approach, we can derive the policies without considering the transition dynamics. 
\item $r_t$: The immediate reward at time step $t$ is $\sum_{i\in\mathbb{I}_t}x_iv_i$ where $\mathbb{I}_t$ is the set of impression opportunities between time step $t$ and $t+1$. 
\item $c$: Similar with the immediate reward definition, the cost between time step $t$ and $t+1$ is $\sum_{i\in\mathbb{I}_t}x_ic_i$.
\item $\gamma$: We set reward(cost) discount factor $\gamma=1$ since the optimization goal of the budget constrained bidding problem is to maximize total reward value under the cost constraint regardless of the reward (cost) time. 
\end{itemize}

Deriving this episodic CMDP setup does not necessarily mean the problem is solved. In fact, the immediate reward function will make the agent converge to the suboptimal solution and the exploration strategies in DQN is not efficient under some circumstances. In the rest of this section, we focus our discussion on the reward function design methodology and the exploration improvements.

\subsection{Reward Function Design}
\subsubsection{The Reward Function Design Trap}
In our problem modeling, the immediate reward $r_t$ is the total value of winning impressions between $t$ and $t+1$. It can be easily inferred that $r_t$ monotonously increases as $\lambda_{t-1}$ decreases (and vice versa) as a small $\lambda_{t-1}$ encourages an aggressive bidding and therefore wins more impressions between step $t-1$ and $t$. In this case, the agent will be obsessed with taking actions to decrease $\lambda$ and finally converge to suboptimal solutions. The reason why the agent tends to converge to suboptimal solutions is two-folded:
\begin{itemize}

\item\textbf{Neglection of budget constraint:} Budget is the critical resource in budget constrained bidding. It is not difficult to imagine that consuming too much budget acquiring impressions at the beginning may not be the optimal strategy. However, the immediate reward does not consider the budget constraint at all. 

\item \textbf{Policy greedy character and poor exploration efficiency:} 
At the beginning of the DQN training process, reward $r_t$ will significantly affect $Q$'s output, which makes $Q$ tend to give relatively higher output to the action receiving larger immediate reward. As a result, the agent will hardly change its inclination because there is no punishment in our formulation\footnote{Existing work in the CMDP framework to resolve this trap is introducing a punishment factor $\alpha$ to integrate the cost into the reward function, i.e., $r'_t$=$r_t+\alpha c_t$. However, deriving the right $\alpha$ can be another cumbersome task since it will take a lot of time to balance $\alpha$ and final performance.}. Further the line, once the agent has the inclination of consuming more budget at the beginning of the episode, it becomes very hard to explore the optimal sequence of actions, because the agent will miss the potentially more valuable impressions later in the episode due to running out of budget.
\end{itemize} 

\subsubsection{A Reward Design Methodology} \label{rf_intro}

It becomes crucial to design a new reward function that can avoid the drawbacks of the immediate reward and is simple enough to boost the agent's convergence to the optimal one. Here ``simple'' means: 1) the reward naturally encodes the constraint, 2) it is easy to implement, and 3) it can be generalized to other scenarios beyond budget constrained bidding.
Let us take a look at the problem from another angle. Intuitively, the return of an entire episode will tell us how good the agent does. We believe this would be a good reward for all the $(s,a)$ pairs in this episode (recall that an episode is a sequence of states, actions, and rewards). In order to relieve the effect brought by other state-action pairs\footnote{The return of an episode is jointly decided by all the $(s,a)$ pairs in it.}, we consider a new reward function design for $(s,a)$ with the form:

\begin{equation} \label{eq:r_def}
\begin{aligned}
\mathtt{r}(s, a) = \max_{e \in E(s,a)} \Sigma_{t=1}^T r_t^{(e)} \\
\end{aligned}
\end{equation}

\noindent where $E(s,a)$ represents the set of existing episodes that the agent took action $a$ at state $s$ and $r_t^{(e)}$ is the original immediate reward at step $t$ within episode $e$. We leverage the episodic nature of the process so that new reward will be continuously updated during the policy optimization. Please note that the methodology can be generalized into other resource-constrained RL problems such as the game Gold Miner\footnote{https://www.crazygames.com/game/gold-miner}, in which the action results in more value in unit time will be more rewarded with our reward design methodology.

One may have the concern that whether the optimal policy with the new reward function is the same as the one with the original immediate reward. Fortunately, they are the same as long as there is only one initial state and the MDP is deterministic.

\begin{theorem} 
Let $\pi^*_\mathtt{r}$ be an optimal policy if the reward function is $\mathtt{r}$ in our MDP formulation. If the deterministic MDP with fixed T steps has only one initial state  $s_1$, $\pi^*_\mathtt{r}$ is guaranteed to be an optimal policy $\pi^*_r$ in the original MDP  formulation with immediate reward $r$.\footnote{We omit both the cost function $c$ and constraint $B$ because the optimal policy must have met the budget constraint.} \label{algo:prove}
\end{theorem}


\begin{proof}
Let $\{(s_i,\pi^*_\mathtt{r}(s_i),\mathtt{r}(s_i,\pi^*_\mathtt{r}(s_i)))\},i=1,..,T$ be an episode produced by policy $\pi^*_\mathtt{r}$ with initial state $s_1$, where $\pi^*_\mathtt{r}(s_i)$ is the action taken at state $s_i$ and $T$ is the episode length. Let $V^\pi_\mathtt{r}$ denote the return obtained when applying  policy $\pi$ to the MDP with reward function $\mathtt{r}$. Since the MDP has only one initial state, it can be guaranteed that
\begin{equation} \label{eq:proof1}
\mathtt{r}(s_i,\pi^*_{\mathtt{r}}(s_i)) \leq V^{\pi^*_{r}}_{r}, \forall i = 1, ..,T. 
\end{equation}
where $V^{\pi^*_{r}}_{r}$ is denoted as the maximal return if the reward function is $r$.
Therefore, we can infer that ${V^{\pi^*_{\mathtt{r}}}_{\mathtt{r}}}$, i.e., the maximal return if the reward function is $\mathtt{r}$, is no larger than $T\cdot{V^{\pi^*_{r}}_{r}}$, i.e.,
\begin{equation} \label{eq:proof2}
{V^{\pi^*_{\mathtt{r}}}_{\mathtt{r}}} = \Sigma^T_{i=1} \mathtt{r}(s_i,\pi^*_{\mathtt{r}}(s_i)) \le T\cdot{V^{\pi^*_{r}}_{r}},
\end{equation}
where equality holds if and only if ${\mathtt{r}(s_i,\pi^*_{\mathtt{r}}(s_i))} = {V^{\pi^*_{r}}_{r}}, \forall i = 1,...,T$.

On the other hand, the episode produced by $\pi^*_{r}$ has the maximal return ${V^{\pi^*_{r}}_{r}}$. Therefore at each state of this episode, the reward $\mathtt{r}$, which is the maximal return of episodes according to Eq.~\eqref{eq:r_def}, is also ${V^{\pi^*_{r}}_{r}}$ ,i.e.,
\begin{equation} 
{V^{\pi^*_{r}}_{\mathtt{r}}} = \Sigma^T_{i=1} \mathtt{r}(s_i,\pi^*_{r}(s_i)) = T\cdot{V^{\pi^*_{r}}_{r}}
\end{equation}

Since $\pi^*_{\mathtt{r}}$ is the optimal policy if the reward function is $\mathtt{r}$, it obtains the maximal value w.r.t. reward $\mathtt{r}$. Thus we have

\begin{equation} \label{eq:proof4}
\begin{aligned}
{V^{\pi^*_{\mathtt{r}}}_{\mathtt{r}}} \ge 
{V^{\pi^*_{r}}_{\mathtt{r}}} 
= T \cdot {V^{\pi^*_{r}}_{r}}.
\end{aligned}
\end{equation}

Based on Eqs. \eqref{eq:proof2} and \eqref{eq:proof4}, we have
${V^{\pi^*_{\mathtt{r}}}_{\mathtt{r}}}= T \cdot {V^{\pi^*_{r}}_{r}}$ and ${\mathtt{r}(s_i,\pi^*_{\mathtt{r}}(s_i))} = {V^{\pi^*_{r}}_{r}}, \forall i=1,..,T$.

This means $\pi^*_{\mathtt{r}}(s_i)$ is also an optimal action if the reward function is $r$ for any state $s_i$. Therefore, the optimal policy with reward function $\mathtt{r}$ is also optimal with reward function $r$.
\end{proof}

\subsection{Adaptive $\epsilon$-greedy Policy}
We use DQN as our model-free RL algorithm and it is default to uses $\epsilon$-greedy policy to balance exploitation and exploration, i.e., the agent chooses action $a^*=\arg\max_a Q(s,a)$ with probability 1-$\epsilon$ or otherwise takes a random action. 
$\epsilon$ is usually initialized with a large value and gradually anneals to a small value over time. However, how to set a proper annealing speed is critical: a high annealing speed usually makes exploration insufficient while a low one usually results in a slow policy convergence.

Fortunately, in the budget constrained bidding problem, the optimal bidding theory guarantees a fixed optimal $\lambda_t^*$ for each step $t \in \{1, ..., T\}$. The optimal action at state $s_t$ is to adjust $\lambda$ as close to $\lambda_t^*$ as possible, otherwise, the more deviating from the optimal action, the more potential value should be got (less $Q$ value). Therefore, based on the action space $\mathcal{A}$ we defined (i.e., a set of adjustment rate $\{\beta_a\}$), the action-value distribution $Q(s_t,a_t)$ over $\mathcal{A}$, sorted by the action's adjustment scale $\beta_a$, should be unimodal, such as the plot illustrated in Fig. \ref{fig:adaptive_f1}. When the distribution is not unimodal, e.g. the plot in Fig. \ref{fig:adaptive_f2}, we believe the current estimation of $Q$ is abnormal and the $\epsilon$ should be increased to encourage more explorations under this state. This simple yet efficient adaptive $\epsilon$-greedy policy usually works well in our  practice with the budget constrained bidding problem.

\begin{figure}
    \centering
    \begin{subfigure}[b]{0.23\textwidth}
        \includegraphics[width=\textwidth]{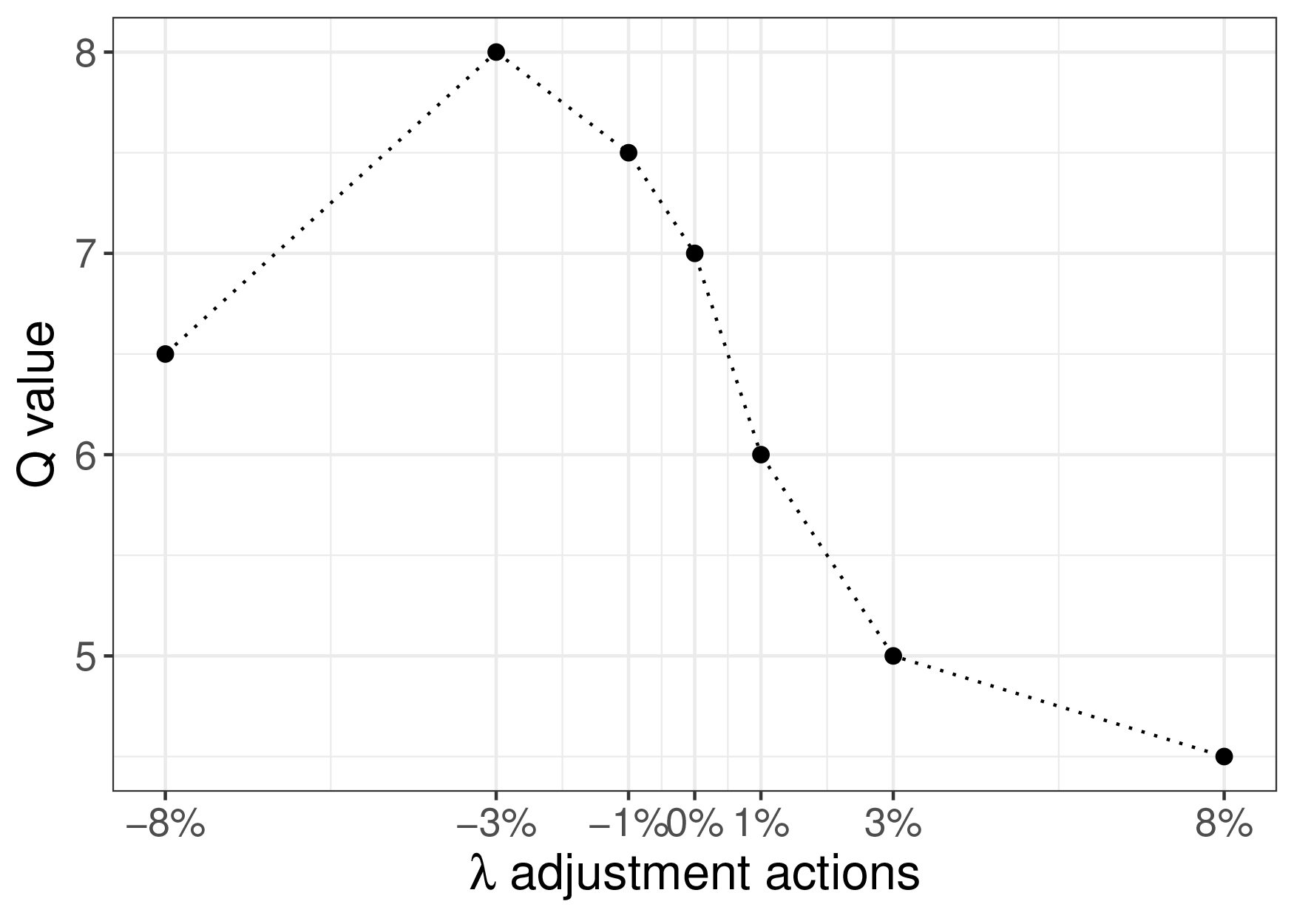}
        \caption{ }

        \label{fig:adaptive_f1}
    \end{subfigure}
    \begin{subfigure}[b]{0.23\textwidth}
        \includegraphics[width=\textwidth]{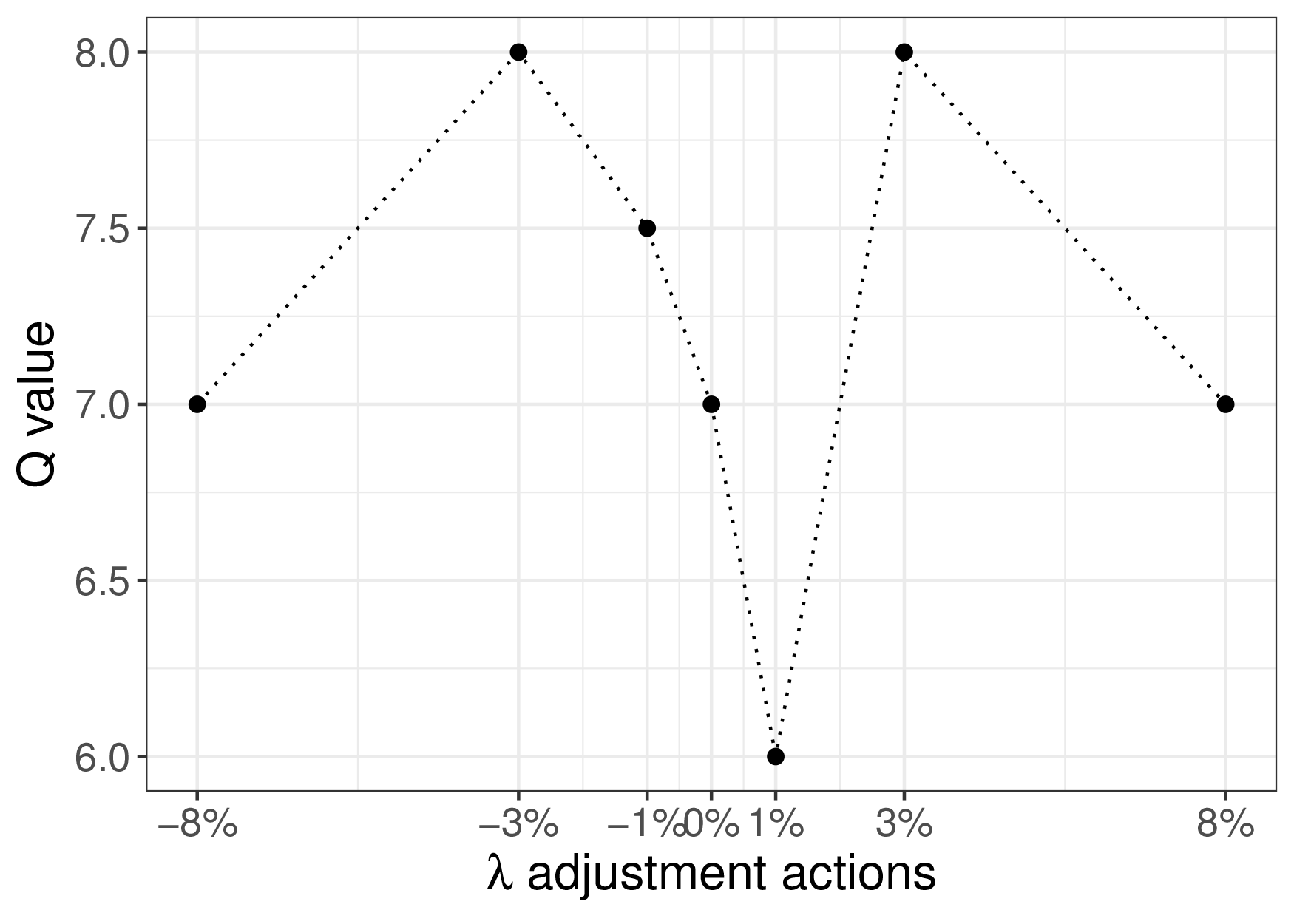}
        \caption{ }

        \label{fig:adaptive_f2}
    \end{subfigure}
    \caption{Distribution examples of action-value $\pmb{Q}$ during training. (a) Normal distribution. (b) Abnormal distribution.}\label{fg:adaptive_f}
\end{figure}

\subsection{Deep Reinforcement Learning to Bid}

Putting them together, we present our Deep Reinforcement Learning to Bid (DRLB) framework.

The framework is built on top of DQN\cite{mnih2015human}, where the state-action value function $Q$ is given by a deep neural network parameterized with $\theta$. The process an agent interacting with the auction system within the DRLB framework can be illustrated in Fig. \ref{fg:5}. Based on the adaptive $\epsilon$-greedy policy, the agent takes an action $a_t \in \mathcal{A}$ (adjusting $\lambda_{t-1}$ to $\lambda_{t}$) under state $s_t \in \mathcal{S}$ at step $t \in \{1, ..., T\}$. Then, bids are produced based on Eq. \eqref{eq:1} with $\lambda_t$ for the advertiser to compete with other bidders. At step $t+1$, the environment returns $\mathtt{r}_{t+1}$ and $s_{t+1}$. 
The agent updates $\theta$ by performing a gradient descent according to the loss calculated based on a mini-batch of $(s, a, s', \mathtt{r}(s,a))$ sampled from experience. The complete DRLB framework is presented in Algo. \ref{algo:2}. 

One thing deserves particular attention is that, different from immediate reward $r_t$, our reward function $\mathtt{r}_t$ is not directly observable from the environment at each step. However, storing $\mathtt{r}(s,a)$ for all the $(s,a)$ pairs is not feasible for large-scale applications such as in industry-level advertising systems. Therefore, we use another deep neural network to predict it and the neural network is called RewardNet in our framework. The RewardNet, parameterized with $\eta$, is simultaneously learned with the $Q$ function in the DRLB framework. The algorithm used to learn this RewardNet is shown in Algo. \ref{algo:1}.


\begin{figure} 
  \centering
    \includegraphics[width=0.35\textwidth]{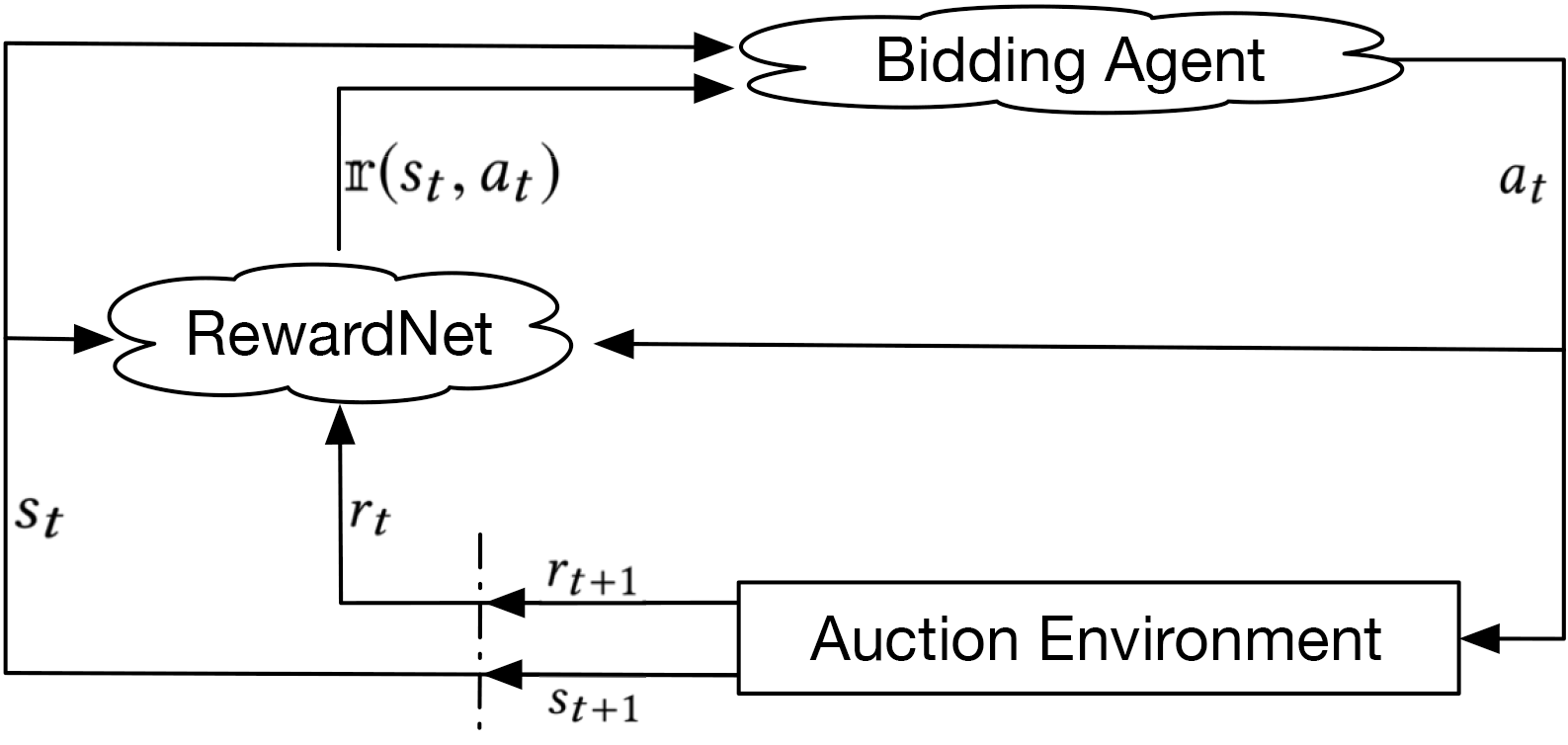}
  \caption{Illustration of Deep Reinforcement Learning to Bid.}
  \label{fg:5}
\end{figure}

\begin{algorithm}
Initialize replay memory $\mathcal{D}_1$ to capacity $N_1$\;
Initialize $Q$ with random weights $\theta$\;
Initialize $Q\ target$ with weights $\theta^- = \theta$\;

\For{$episode = 1$ \KwTo $K$}{
Initialize $\lambda_0$\;
Bid with $\lambda_0$ according to Eq. \eqref{eq:1}\;
\For{$t = 1$ \KwTo $T$}{
Update \textbf{RewardNet} (8-10 in Algo. \ref{algo:1})\;
Observe state $s_t$\;
Get action $a_t$ from \textbf{adaptive $\boldsymbol{\epsilon}$-greedy policy}\;
Adjust $\lambda_{t-1}$ to $\lambda_t$\;
Bid with $\lambda_t$ according to Eq. \eqref{eq:1}\;
Get $r_{t}$ from \textbf{RewardNet}\;
Observe next state $s_{t+1}$\;
Store $(s_t, s_{t+1}, a_t, r_{t})$ in $\mathcal{D}_1$\;
Sample mini-batch of $(s_j, s_{j+1}, a_j, r_{j})$ from $\mathcal{D}_1$\;
\uIf{$s_{j+1}$ is the terminal state}{
Set $y_j=r_j$\;
}
\Else{
Set $y_j=r_j + \gamma \max_{a'} Q(s_{j+1}, a'; \theta^-)$\;
}
Perform a gradient descent step on $(y_j-Q(s_j,a_j;\theta))^2$ with respect to $\theta$\;
Every C steps reset $Q\ target=Q$\;
}
Store data for \textbf{RewardNet}\;
}
\caption{Deep Reinforcement Learning to Bid}\label{algo:2}
\end{algorithm}

\begin{algorithm}
Initialize replay memory $\mathcal{D}_2$ to capacity $N_2$\;
Initialize reward network $\mathcal{R}$ with random weights $\eta$\;
Initialize reward dictionary $\mathcal{M}$ to capacity $N_3$\;
\For{$episode = 1$ \KwTo $K$}{

Initialize temporary set $\mathcal{S}$\;
Set $V$ = 0\;
\For{$t = 1$ \KwTo $T$}{
\uIf{$len(\mathcal{D}_2) >$ BatchSize}{
Sample mini-batch of $(s_j, a_j, \mathcal{M}(s_j,a_j))$ from $\mathcal{D}_1$\;
Perform a gradient descent step on $(\mathcal{R}(s_j,a_j;\eta)-\mathcal{M}(s_j,a_j))^2$ with respect to the network parameters $\eta$\;
}
Observe state $s_{t}$\;
RL agent executes $a_t$ in the Environment\;
Obtain immediate reward $r_{t}$ from the Environment\;
Set $V = V + r_{t}$\;
Store pair $(s_t, a_t)$ in $\mathcal{S}$\;
}
\For{($s_j,a_j$) \textbf{in} $\mathcal{S}$}{
Set $\mathcal{M}(s_j,a_j) = \max (\mathcal{M}(s_j,a_j), V)$\;
Store pair $(s_j, a_j, \mathcal{M}(s_j, a_j))$ in $\mathcal{D}_2$\;
\uIf{$|\mathcal{M}| > N_3$}{
Discard old key in $\mathcal{M}$ based on LRFU strategy \cite{lee2001lrfu}\;
}
}
}
\caption{Learning RewardNet}\label{algo:1}
\end{algorithm}

\section{Experimental Results}

In this section, we present the empirical study of DRLB. First, the experimental setup and implementation details of DRLB are introduced. Then, we quantitatively compare DRLB with several baseline methods and the state-of-the-art method RLB\cite{cai2017real} on two large-scale datasets. Finally, we investigate the effectiveness of the reward function design and the adaptive $\epsilon$-greedy policy.

\subsection{Experimental Setup}

\subsubsection{Dataset}
We investigate the performance of DRLB on two datasets. \emph{Dataset A} is from a leading e-commerce advertising platform in China. The dataset comprises 2 billion impressions and their predicted CTR from 10 continuous days in Jan 2018. \emph{Dataset B} is from iPinYou\footnote{iPinYou dataset is available at http://data.computational-advertising.org}, augmented with the predicted CTR produced by the same model used in \cite{cai2017real}. For both of the two datasets, the first 7 days of data are used for training and the last 3 days of data are used for evaluation. Each day comprises an episode and each episode consists of 96 time steps (15 minutes between two consecutive time steps).

It is worth noting that in \emph{Dataset A}, the bidding environment may change significantly on a daily basis. For instance, the advertisers of a certain e-commerce category, e.g. clothing, usually have to deal with traffic bursts and highly competitive bids from other bidders on certain holidays or festivals, such as Women's Day. The campaign settings such as budget are usually adjusted by the advertisers according to such situations as well. 

\subsubsection{Evaluation Metrics} \label{section:metric}

The goal of the budget constrained bidding is to maximize the total value of winning impressions. The impression value is usually associated with the expectation of the desired outcomes such as ad clicks or conversions. The evaluation metric can be defined as the total predicted CTR or the total real clicks of winning impressions. 

In \emph{Dataset A}, the auction log contains both the winning impressions and the lost impressions. Since the click event is not available for those lost impressions,  we consider using predicted CTR as the impression value. Based on the predicted CTR $v_i$ and cost $c_i$ for all candidate impressions at the end of the episode, it is easy to get the theoretically best result $R^*$ using the optimal $\lambda^*$ calculated with the greedy approximation algorithm \cite{dantzig1957discrete}. Thus, the difference between $R^*$ and the total value of the winning impressions $R$ under the current policy is a simple and effective metric to evaluate the policy, denoted by $R/R^*$. In \emph{Dataset B}, the click results for all the candidate impressions are known, so that the real click number from winning impressions is a proper metric to evaluate different methods. 






\subsubsection{Baseline Methods}
We compare DRLB with the widely used methods in the industry as well as the state-of-the-art method.

\textbf{1)\ Fixed Linear Bidding (FLB)} uses a fixed $\lambda_0$ to linearly bid according to Eq. \eqref{eq:flb}, which is very straightforward and is widely used in industry.

\begin{equation} 
bid=v_i/\lambda_0 \label{eq:flb}
\end{equation}

\textbf{2)\ Budget Smoothed Linear Bidding (BSLB)~\cite{hegeman2011budget}} gives a practical way of bidding under budget constraint. It combines the classic bidding Eq. \eqref{eq:flb} with the current budget consumption information $\Delta$, which equals to episode time left ratio divided by budget left ratio. When the budget left ratio is lower than the time left ratio, the bid is decreased by adjusting the $\lambda_0$ downward, otherwise, the bid is increased to consume more budget.

\begin{equation} \label{eq:7}
\begin{aligned}
bid = v_i / (\lambda_0 * \Delta)\\
\end{aligned}
\end{equation}

\textbf{3)\ Reinforcement Learning to Bid (RLB)~\cite{cai2017real}} is the state-of-the-art algorithm for budget constrained bidding. RLB formalized the auction process as a MDP where the agent needs to take action, i.e., provide a bid, for every impression opportunity. The agent is trained with a model-based RL approach to maximize total value of winning impressions under a certain budget constraint. 

\subsubsection{Implementation Details of DRLB}

In DRLB, we take a fully connected neural network with 3 hidden layers and 100 nodes for each layer as our state-action value function $Q$ and another identically structured neural network as the RewardNet. The mini-batch size is set to 32 and the replay memory size is set to 100,000. The agent has 7 candidate actions corresponding to 7 different $\lambda$ adjustment rates -8\%, -3\%, -1\%, -0\%, 1\%, 3\%, and 8\% respectively. $\lambda$ is adjusted every time step. The initial value of $\epsilon$ in $\epsilon$-greedy policy is set to 0.9 and final value is set to 0.05. The $\epsilon$ at each step $t$ is empirically set as $\max(0.95 - r_\epsilon\times t,\ 0.05)$, where $t$ is the time step number and $r_\epsilon$ is the parameter controlling the annealing speed. With the adaptive $\epsilon$-greedy policy, if the distribution of action-value $Q$ is not unimodal w.r.t. the sorted adjustment rates, the agent randomly chooses an action with probability $\epsilon=\max(\epsilon, 0.5)$. Following the common practice of DQN, we set every $C=100$ steps $Q\ target$ updates $\theta^-$ with $\theta$ and the learning rate is set to 0.001 and momentum is set to 0.95. 



%

\subsection{Evaluation Results on Dataset A}

We first conduct experiments to compare the performance of FLB, BSLB,RLB, and DRLB on \emph{Dataset A}. Please note that FLB, BSLB and DRLB require an initial $\lambda_0$ at the beginning of each episode. A heurestic to derive $\lambda_0$ is to use the theorectially optimal $\lambda^*_{prev}$ of the previous episode\footnote{Since we have the all the knowledge of the auction information in the experiment data, we are able to derive the theoretically optimal $\lambda^*$.}. Due to the variance of auction environment and campaign settings, the derived $\lambda_0$ may deviate from the theoretically optimal $\lambda^*$ of the current episode. To investigate the performance of each method with different $\lambda$-deviations, quantified by $(\lambda_0 -  \lambda^*)/\lambda^*$, we divide all campaigns into 9 groups according to the $\lambda$-deviation and evaluate the four methods in each group.

The experimental results are summarized in Table \ref{table:2}. DRLB almost outperforms FLB, BSLB and RLB in all 9 groups and the overall improvements over FLB, BSLB and RLB is 100.92\%, 18.33\% and 16.80\% respectively. Moreover, as the $\lambda$-deviation increases, the performance of all baselines degrades enormously while DRLB can still obtain desirable performance. In the cases that the $\lambda$-deviation is small, e.g.[ -20\%, 0\%), all methods can achieve decent performance. In the groups where $\lambda$-deviation is large, DRLB shows particular advantage over the baselines, which indicates its superior adaptability even starting with a improper $\lambda_0$. For instance, when the $\lambda$-deviation lies in [-100\%, -80\%), the average $R/R^*$ of DRLB is 0.878 while that of FLB, BSLB and RLB is 0.436, 0.525 and 0.430 respectively. 

\begin{table*}
\caption{The $\pmb{R/R^*}$ improvements of DRLB over other three methods in 9 groups of $\pmb{\lambda}$ deviation based on Dataset A.}
\centering
\scalebox{0.9}{
    \begin{tabular}{lc|cccc|rrr}
        \toprule
\multirow{2}{*}{$\lambda$ Deviation}  &  \multirow{2}{*}{Campaigns}  &  \multicolumn{4}{c|}{$R/R^*$}  &  \multicolumn{3}{c}{Improvements of $R/R^*$}    \\  
        &  &  {FLB}  &  {BSLB}  &  {RLB} & {DRLB\centering} &  {FLB}  &  {BSLB} &  {RLB} \\
        \hline
$\left[-100\%,-80\%\right)$  &  43  &  0.436  &  0.525 & 0.430  &  {\bfseries 0.878}    &  101.38\%  &  67.24\% & 104.19\%  \\
$\left[-80\%,-40\%\right)$  &  89  &  0.434  &  0.647 & 0.800  &  {\bfseries 0.884}    &  103.69\%  &  36.63\% & 10.50\%  \\
$\left[-40\%,-20\%\right)$  &  66  &  0.697  &  0.901 & 0.927  &  {\bfseries 0.945}    &  35.58\%  &  4.88\% & 1.94\%  \\
$\left[-20\%,0\%\right)$    &  41  &  0.863  &  0.936 & {\bfseries 0.965}  &   0.953    &  10.43\%  &  1.82\% & -1.24\%  \\
$\left[0\%,20\%\right)$  &  39  &  0.825  &  0.925  & 0.944   &  {\bfseries 0.950}  & 15.15\%  &  2.70\% & 0.64\%  \\
$\left[20\%,40\%\right)$  &  48  &  0.491  &  0.947  & 0.895   &  {\bfseries 0.948}  &  93.08\%  &  0.11\% & 5.92\%  \\
$\left[40\%,80\%\right)$  &  85  &  0.391  &  0.904  & 0.832   &  {\bfseries 0.928}  &  137.34\%  &  2.65\% & 11.54\%  \\
$\left[80\%,160\%\right)$  &  57  &  0.307  &  0.813  & 0.709   &  {\bfseries 0.924}  &  200.98\%  &  13.65\% & 30.32\%  \\
$\left[160\%,\infty\right)$  &  32  &  0.291  &  0.668  & 0.618    &  {\bfseries 0.904}   &  210.65\%  &  35.33\% & 46.28\%  \\
        \hline
        Average  &  &  0.526  &  0.807  & 0.791 &  {\bfseries 0.924}  &  100.92\%  &  18.33\% & 16.80\%  \\
        \bottomrule
    \end{tabular}}
    \label{table:2}
\end{table*}

We analyse the experimental results as follows. Because the environment from testing data may deviate from training data heavily, some statistics about the auction environment in the training data, such as budget and market price distribution, may be invalid in the testing data. FLB gives the worst performance since it uses a fixed $\lambda_0$ and is unaware of the dynamics of the auction environment. Similarly, RLB is also insensitive to the variance of the auction environment, because it assumes the market price distribution is stationary and hence the bidding policy is calculated based on this market price distribution.

Both BSLB and DRLB show the ability to cope with the environment changes. However, DRLB is even better than BSLB due to the following two reasons. First, BSLB only takes into consideration budget and time, while DRLB could make full use of auction information to enable accurate $\lambda$ control. Second, according to Eq. \eqref{eq:7}, BSLB is insensitive to the time elapse at the early stages of the day and thus shows limited adaptability to the environment. As for DRLB, it is always able to make timely reaction because the state can represent the environment comprehensively. 

\subsection{Evaluation Results on Dataset B}

We further compare DRLB with RLB on \emph{Dataset B}. The results are shown in Table \ref{table:exp1}. We can see that 5 out of 9 campaigns observe improvements in terms of acquired clicks if the bidding strategy is DRLB. The overall improvement is 4.3\%. We are particularly interested in those campaigns that DRLB performs worse than RLB. A straightforward observation is that DRLB usually performs worse than RLB on campaigns with low AUCs. AUC is a popular indicator to measure the CTR prediction accuracy. A low AUC usually suggests a poor CTR prediction accuracy. Remember that DRLB bids linearly with the predicted impression value, which is quantified by the predicted CTR. Therefore it is not difficult to understand the suboptimal performance of DRLB on campaigns with poor CTR predictions. We argue that improving CTR prediction should be a separate effort (and this is usually the practice in real advertising systems) and the performance of DRLB can be directly improved if the CTR prediction is improved. 



\begin{table}
\caption{Detailed AUC and real clicks for DRLB and RLB (T = 1000 and ${c_0}$ = 1/16) in Dataset B.}
\centering
\scalebox{0.9}{
    \begin{tabular}{c|c|ccr}
        \toprule
Campaign  &  AUC  &  RLB  &  DRLB &  Improvements    \\  
        \hline
1458    & 97.73\% &   473    &    \textbf{474}   &    0.2\% \\
2259    & 67.90\% &    \textbf{23}    &    22   &    -4.3\%  \\
2261    & 62.16\% &    \textbf{17}    &    15   &    -11.8\% \\
2821    & 62.95\% &    \textbf{66}    &    \textbf{66}  &    0\%  \\
2997    & 60.44\% &    \textbf{119}    &    117   &    -1.6\% \\
3358    & 97.58\% &    219    &    \textbf{225}   &    2.7\% \\
3386    & 77.96\% &    109    &    \textbf{134}   &    22.9\% \\
3427    & 97.41\% &    307    &    \textbf{310}   &    1.0\% \\
3476    & 95.84\% &    203    &    \textbf{239}   &    17.7\%  \\
\hline
Overall & -&1536 &\textbf{1602} & 4.3\%\\
        \bottomrule
    \end{tabular}}
    \label{table:exp1}
\end{table}

\subsection{Convergence Comparison with Immediate Reward Function}

In order to probe the agent behaviors with different reward function, i.e. RewardNet and immediate reward, we train two models independently on \emph{Dataset A}. One model deploys RewardNet, while the other uses immediate reward. We dump the models every 10 episodes during the training process, and compare them on the testing dataset. As illustrated in Fig. \ref{fig:reward_net_converge_cmp}, the model with RewardNet yields satisfying $R/R^*$ of 0.893 within a small number of steps, while the model with immediate reward gets stuck in some inferior policy and yields poor $R/R^*$ of 0.418. The experimental results indicate the effectiveness of RewardNet in leading the agent to the optimal policy. 

Furthermore, to represent the reward function design trap, we depict a typical case of the immediate reward distribution, i.e. $r_t/R^*$ over time step $t \in [1, ..., T]$, of agents trained with different reward function, which is shown in Fig. \ref{fig:reward_dist}. The ideal immediate reward distribution derived by the theoretically optimal $\lambda^*$ is also illustrated as a reference. The results demonstrate that the model trained with immediate reward is prone to obtaining more immediate reward in the early steps, which exhausts the budget and results in poor performance from a long-term view. In the case of RewardNet, the reward distribution is similar to that with $\lambda^*$, which shows that RewardNet helps the agent avoid greedy behavior and better utilize the budget for overall benefits.

\subsection{Effectiveness of the Adaptive $\epsilon$-greedy Policy}


Experiments are performed to compare our adaptive $\epsilon$-greedy policy with the original $\epsilon$-greedy policy. We evaluate these policies in settings with two different annealing rates, i.e. $r_\epsilon$=2e-5 and $r_\epsilon$=1e-5 respectively. The results shown in Fig. \ref{fg:4} demonstrate that our adaptive exploration policy helps the agent explore effectively and achieve better performance in both settings. Specially, the adaptive $\epsilon$-greedy exploration enables fast convergence and significantly outperforms the original $\epsilon$-greedy in the setting with the higher annealing rate. This indicates the superior efficiency of our exploration strategy in circumstance where training time is limited, which is common in reinforcement learning problems.

\begin{figure}
    \centering
    \begin{subfigure}[b]{0.35\textwidth}
        \includegraphics[width=\textwidth]{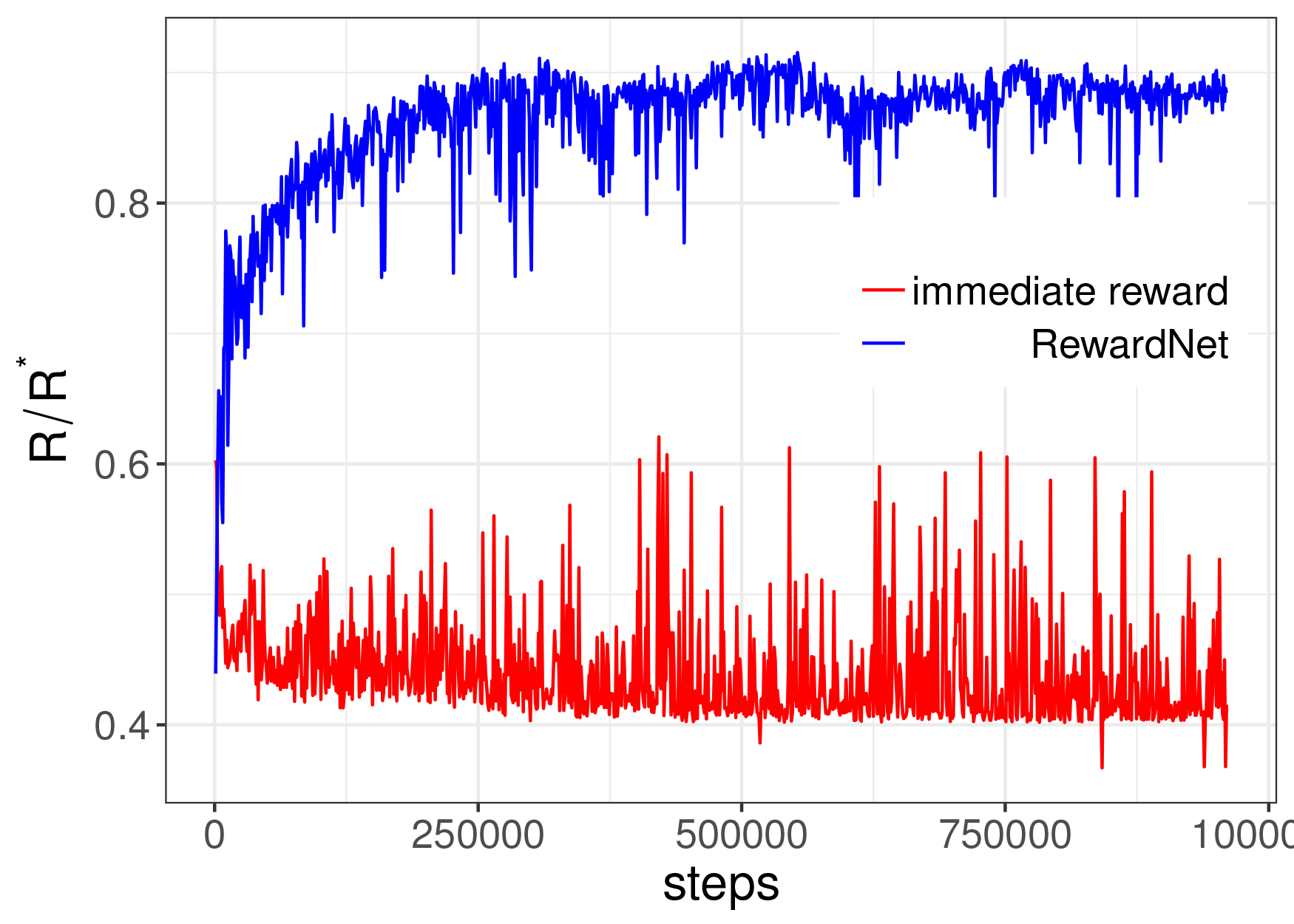}
        \caption{ }

        \label{fig:reward_net_converge_cmp}
    \end{subfigure}
    \begin{subfigure}[b]{0.35\textwidth}
        \includegraphics[width=\textwidth]{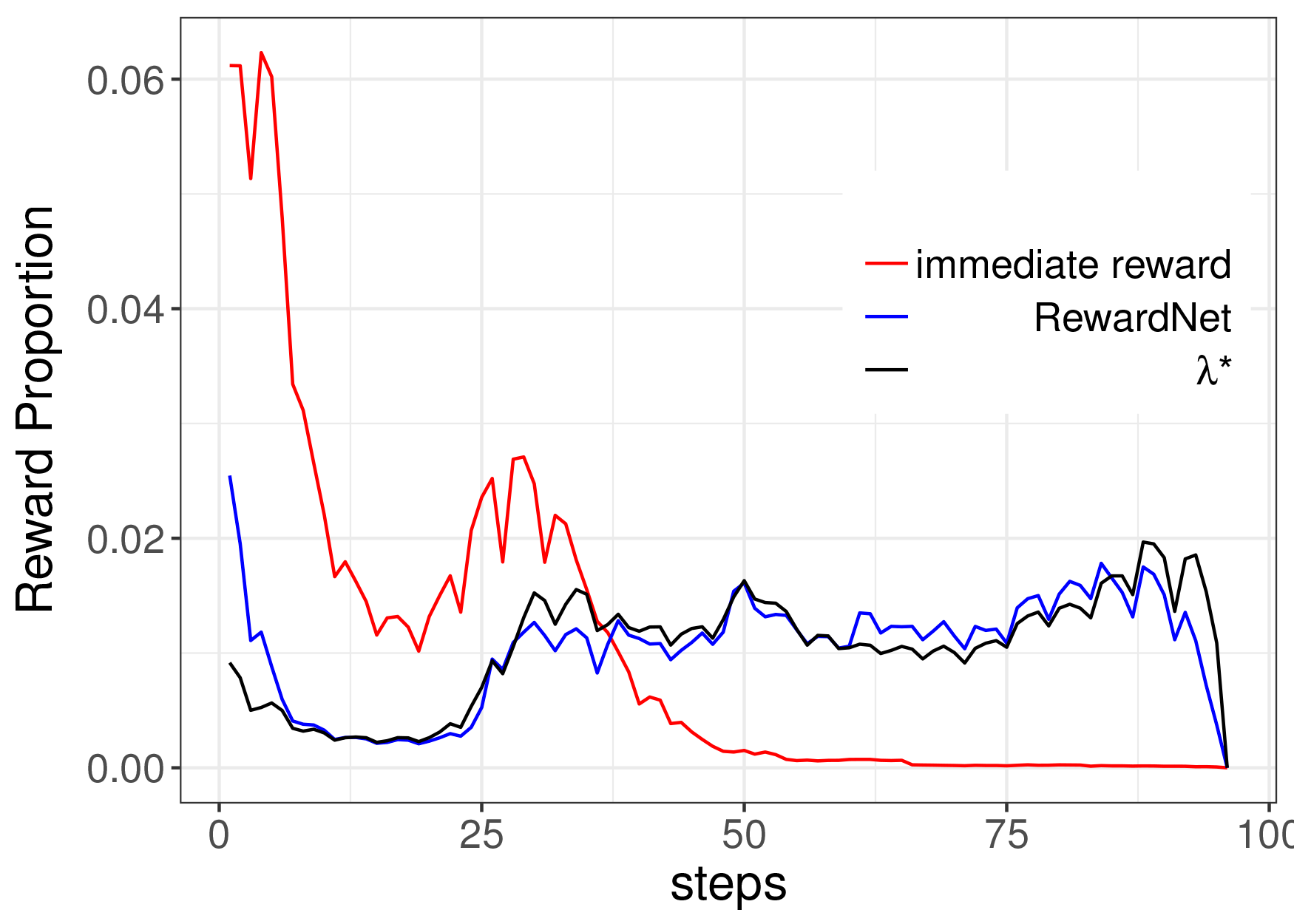}
        \caption{ }

        \label{fig:reward_dist}
    \end{subfigure}
    \caption{Comparison between RewardNet and immediate reward. (a) The $\pmb{R/R^*}$ of two models over steps. (b) Reward distribution of two models along with the ideal one in an episode.}\label{fg:3}
\end{figure}

\begin{figure}
    \centering
    \begin{subfigure}[b]{0.35\textwidth}
        \includegraphics[width=\textwidth]{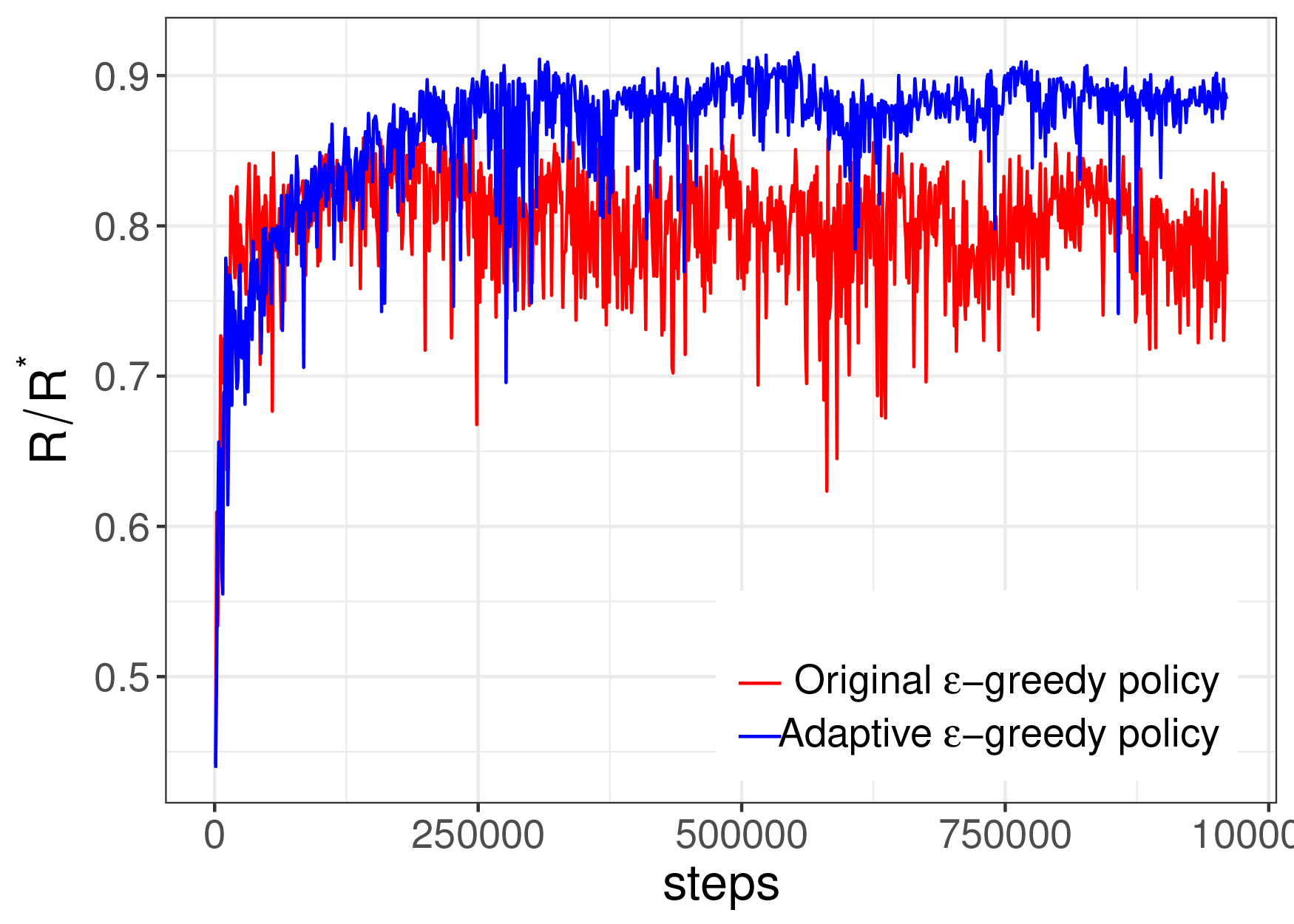}
        \caption{ }
        \label{fig:eg_fast}
    \end{subfigure}
    \begin{subfigure}[b]{0.35\textwidth}
        \includegraphics[width=\textwidth]{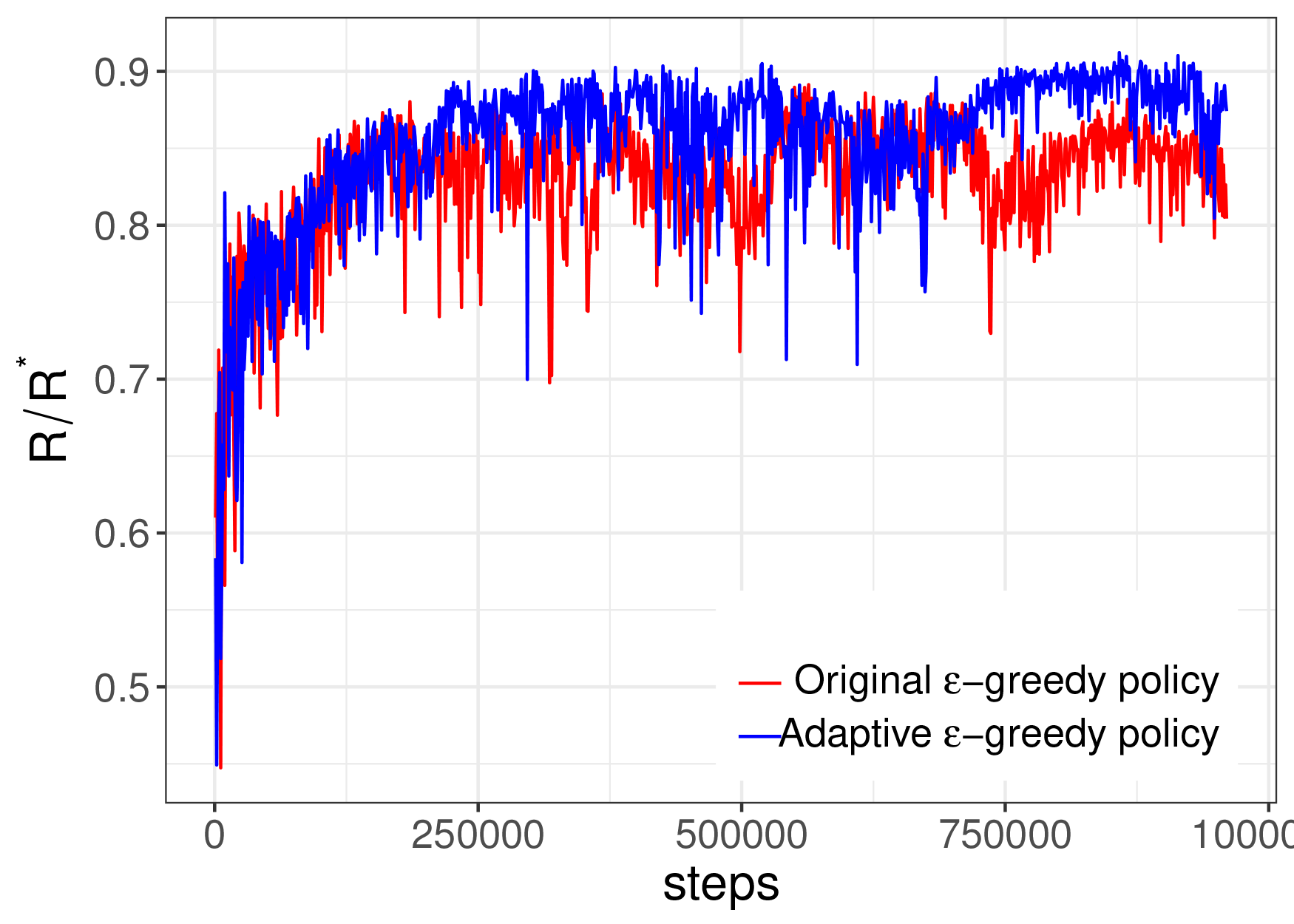}
        \caption{ }
        \label{fig:eg_slow}
    \end{subfigure}
    \caption{Performance of adaptive $\pmb{\epsilon}$-greedy and original $\pmb{\epsilon}$-greedy. (a) $\pmb{r_\epsilon}$=2e-5. (b) $\pmb{r_\epsilon}$=1e-5. }\label{fg:4}
\end{figure}

\section{Related Work}
		In RTB display advertising, there has been some work proposed to estimate impression value, e.g. click-through rate (CTR)\cite{mcmahan2013ad} and conversion rate (CVR)\cite{lee2012estimating}, which helps to bid in the impression level. The optimal strategy of advertisers is to bid truthfully according to the estimated impression value under the second price auction\cite{krishna2009auction}. However, truthful bidding may deliver poor performance considering repeated auctions and budget constraints in real-world applications\cite{zhang2014optimal}. For instance, an advertiser may run out of the budget so early in a day and miss the potentially valuable impressions afterwards by thuthful bidding. Perlich et al. \cite{perlich2012bid}, Zhang et al. \cite{zhang2014optimal} and Cai et al. \cite{cai2017real} proposed to optimize the bidding strategy under budget constraints to maximize the accumulated impression value on behalf of advertisers in display advertising scenario. Static bid optimization frameworks proposed in \cite{perlich2012bid} and \cite{zhang2014optimal} set bids according to the static distribution of input data, which cannot work well when the real data distribution deviates from the assumed one. Cai et al. \cite{cai2017real} proposed a reinforcement learning approach that shows robustness to the non-stationary auction environment, which shares some common thoughts with our work. We both formulated the bidding process as a reinforcement learning problem. However, Cai et al. \cite{cai2017real} modeled the state transition via auction competition and derived the optimal policy to bid for each impression on a model-based MDP, which leads to massive computations when datasets get large. In our work, we transformed the original bidding process to $\lambda$ regulating, and proposed the model-free MDP to derive the optimal policy for bidding. 

There is also some work addressing bidding problem in situations different from ours. Amin et al. \cite{amin2012budget} and Yuan et al. \cite{yuan2012sequential} proposed model-based MDPs to set bids in sponsored search, where the decision is made on key-word level. Ghosh et al. \cite{ghosh2009adaptive} optimized the bidding strategy to guarantee a given number of impressions with a given budget. Approaches proposed in  \cite{lee2013real}, \cite{borgs2007dynamics}, \cite{abrams2007optimal} and \cite{mehta2007adwords} adjust the pre-set bid to smooth the budget spending, which helps advertisers to reach a wider range of audience accessible throughout a day and have a sustainable impact. Moreover, some previous work provide insights on the bidding mechanism design for the advertising platform\cite{balseiro2017budget,balseiro2017learning,conitzer2017multiplicative}, while our work focuses on the benefits of the advertisers and aims to optimize their bidding results.

\section{Conclusion}

In this paper, we propose a model-free deep reinforcement learning method to solve the budget constrained bidding problem in RTB display advertising. The bidding problem is innovatively formulated as a $\lambda$ control problem based on linear bidding equation. To solve the reward design trap, which makes the agent hard to converge to the optimum, we design RewardNet to generate reward instead of using the immediate reward. Furthermore, the problem of insufficient exploration is also alleviated by dynamically changing the random probability of the original $\epsilon$-greedy policy.  The experiments upon the real-world dataset show that our model converges quickly and significantly outperforms the widely used bidding methods. Last but not least, the idea of RewardNet is general, which can be applied to other deterministic MDP problems, especially for those aiming to maximize long-term result when the reward function is hard to design.

\bibliographystyle{ACM-Reference-Format}
\balance
\bibliography{sample-bibliography}

\end{document}